%% file: main.tex
\documentclass[11pt]{article}
\usepackage[utf8]{inputenc}
\usepackage{algorithm}
\usepackage{algorithmic}
\usepackage{amsthm}
\usepackage{amssymb}
\usepackage{bbm}
\usepackage{amsmath}
\usepackage{graphicx}
\usepackage{xcolor}
\usepackage{hhline}
\usepackage{hyperref}
\usepackage{breakcites}

\advance \oddsidemargin by -0.8in
\newcommand{\myparskip}{3pt}
\parskip \myparskip

\title{Projection-free Adaptive Regret with Membership Oracles}
\usepackage{times}

\author{
  Zhou Lu\thanks{Google AI Princeton} \thanks{Princeton University}\\
  \and
  Nataly Brukhim\footnotemark[1] \footnotemark[2]\\
  \and 
  Paula Gradu\thanks{UC Berkeley}
  \and
  Elad Hazan\footnotemark[1] \footnotemark[2]\\
}

\usepackage[utf8]{inputenc}
\usepackage{algorithm,algorithmic}

\include{header}
\begin{document}
\maketitle

\begin{abstract}
In the framework of online convex optimization, most iterative algorithms require the computation of projections onto convex sets, which can be computationally expensive. To tackle this problem \cite{hazan2012projection} proposed the study of  projection-free methods that replace projections with  less expensive computations. The most common approach is based on the Frank-Wolfe method, that uses linear optimization computation in lieu of projections. Recent work by \cite{garber2022new} gave sublinear adaptive regret guarantees with projection free algorithms based on the Frank Wolfe approach. 

In this work we give projection-free algorithms that are based on a different technique, inspired by \cite{mhammedi2022efficient}, that replaces projections by set-membership computations.  We propose a simple lazy gradient-based algorithm with a Minkowski regularization that attains near-optimal adaptive regret bounds. For general convex loss functions we improve previous adaptive regret bounds from $O(T^{3/4})$ to $O(\sqrt{T})$, and further to tight interval dependent bound $\tilde{O}(\sqrt{I})$ where $I$ denotes the interval length. For strongly convex functions we obtain the first poly-logarithmic adaptive regret bounds using a projection-free algorithm. 

\end{abstract}

\input{1-intro}

\input{2-prelim}

\input{3-algorithm}
\input{analysis}

\input{4-adaptive}
\input{smooth}

\input{conclusion}


\bibliography{bib}
\bibliographystyle{alpha}

\newpage
\appendix

\input{A-appendix}
\input{B-set-smoothing}
\end{document}

%% file: header.tex
\usepackage{amsmath,amsfonts,amssymb}
\usepackage{mathtools}
\usepackage{amsthm} 
\usepackage{latexsym}
\usepackage{relsize}

\usepackage[capitalise]{cleveref}
\usepackage{xcolor}
\usepackage{dsfont}

\usepackage{multirow}

\newcommand{\A}{\mathcal{A}}

\newcommand{\K}{\ensuremath{\mathcal K}}

\def\regret{\mbox{{Regret}}}

\newcommand{\ignore}[1]{}

\newcommand{\email}[1]{\texttt{#1}}

\theoremstyle{plain}
\newtheorem{theorem}{Theorem}
\newtheorem{lemma}[theorem]{Lemma}

\newtheorem{assumption}{Assumption}

\newtheorem*{theorem*}{Theorem}
\newtheorem*{lemma*}{Lemma}
\newtheorem*{corollary*}{Corollary}
\newtheorem*{proposition*}{Proposition}
\newtheorem*{claim*}{Claim}
\newtheorem*{fact*}{Fact}
\newtheorem*{observation*}{Observation}
\newtheorem*{assumption*}{Assumption}

\theoremstyle{definition}
\newtheorem{definition}[theorem]{Definition}
\newtheorem{remark}[theorem]{Remark}

\newtheorem*{definition*}{Definition}
\newtheorem*{remark*}{Remark}
\newtheorem*{example*}{Example}

 \theoremstyle{plain}
\newtheorem*{theoremaux}{\theoremauxref}
\gdef\theoremauxref{1}

\DeclareMathAlphabet{\mathbfsf}{\encodingdefault}{\sfdefault}{bx}{n}

\DeclareMathOperator*{\argmin}{arg\,min}
\DeclareMathOperator*{\argmax}{arg\,max}

\def\mA{{\mathcal A}}

\renewcommand{\O}{O}

\newcommand{\E}{\mathbb{E}}

\newcommand{\reals}{\mathbb{R}}

\renewcommand{\leq}{~\le~}
\renewcommand{\geq}{~\ge~}

\let\oldtfrac\tfrac
\renewcommand{\tfrac}[2]{\smash{\oldtfrac{#1}{#2}}}

\let\nablaold\nabla
\renewcommand{\nabla}{\nablaold\mkern-2.5mu}

%% file: 1-intro.tex
\section{Introduction}

We consider the problem of efficient learning in changing environments as formalized in the framework of minimizing adaptive regret in online convex optimization \cite{hazan2009efficient}. We focus on settings in which efficient computation is paramount, and computing projections is expensive.  

This setting of projection-free adaptive regret minimization was recently considered in the work of \cite{garber2022new}. The latter paper gives efficient projection-free algorithms that guarantee $O(T^{3/4})$ adaptive regret.  Their algorithm is based on the Frank-Wolfe method, that replaces projections by linear optimization computations.

We consider exactly the same problem, but use a different approach for projection free algorithms. 
Inspired by the recent work of \cite{mhammedi2022efficient}, we consider adaptive regret minimization using set membership computations. Set membership computation amounts to the decision problem of  whether an input point $x$ is in the decision set $\K$ or not. For a comprehensive comparison of the complexity of computing a membership oracle and linear optimization  see \cite{mhammedi2022efficient}. 


For this problem, we propose a simple gradient based  algorithm that achieves improved, and tight, adaptive regret bounds for general convex and strongly-convex loss functions, detailed in the following sub-section.

\subsection{Summary of results}
Our main algorithm is a novel lazy online gradient descent (OGD) algorithm with a specialized regularization function which we call the Minkowski regularization, a  variation of the Minkowski functional.
$$
\gamma(x) = \inf \{ c \geq 1  \ : \  \frac{x}{c} \in \K  \}.
$$

Assuming the decision set $\K$ contains a ball centered at origin, we prove that our algorithm is able to achieve optimal  $O(\sqrt{T})$ and $O(\log T)$ regret bounds for general convex and strongly-convex loss functions respectively using only $O(d\log T)$ calls to the membership oracle per round, as a first step to adaptive regret.

Observing that OGD-based algorithms are easier to adapt for minimizing adaptive regret as compared to Frank-Wolfe type algorithms, we show that our algorithm improves previously known best adaptive regret bounds in the projection-free setting when combined with previous adaptive regret meta-algorithms. We derive an $O(\log^2 T)$ adaptive regret bound for strongly-convex loss functions, and an $\tilde{O}(\sqrt{I})$ strongly adaptive regret bound for general convex loss functions where $I$ denotes the length of any sub-interval. Our results in comparison to previous work are summarized in table \ref{table:result_summary}. Besides the improved regret bounds, our algorithm itself is simple.

As our main technical contribution, it's shown that the smoothness of $\K$'s boundary can yield a better approximation algorithm to $\nabla \gamma(x)$ whose error has a good uniform upper bound in contrast to previous general methods, by exploiting its special structure. This property is used to prove that our original algorithm already guarantees $O(\sqrt{T})$ adaptive regret for general convex loss functions when $\K$'s boundary is smooth.




\begin{table}[ht]
\begin{center}
\begin{tabular}{@{}p{\textwidth}@{}}
	\centering
			\bgroup
\def\arraystretch{1.5}

\begin{tabular}{|c|c|c|c|c|}
		\hline
		  \multirow{2}{*}{Algorithm}        & Regret & Regret      & \multirow{2}{*}{Adaptive?}     & \multirow{2}{*}{Projection-free?}  \\
		         &  (convex) &  (strongly-convex)     &      &  
\\
\hline
\cite{hazan2009efficient}   &  $\tilde{O}(\sqrt{T})$& $O( \log^2 T)$ &$\checkmark$ &$\times$\\
\hline
\cite{cutkosky2020parameter}, \cite{lu2022adaptive} &    $ \tilde{O}(\sqrt{|I|})$ & $\times$ &$\checkmark$ &$\times$\\
\hline
\cite{hazan2012projection} &$O(T^{\frac{3}{4}})$ & $\times$ &$\times$ &$\checkmark$, LOO \\
\hline
\cite{kretzu2021revisiting}, \cite{wan2021projection} &$\times$ & $O(T^{\frac{2}{3}})$ &$\times$ &$\checkmark$, LOO \\
\hline
\cite{mhammedi2022efficient} &$O(\kappa \sqrt{T})$ & $O(\kappa \log T)$ &$\times$ &$\checkmark$, MO \\
\hline
\cite{garber2022new} & $O(T^{\frac{3}{4}})$ &$\times$ &$\checkmark$ &$\checkmark$, LOO\\
\hline
This paper & $ \tilde{O}(\kappa \sqrt{|I|})$ &$O(\kappa \log^2 T)$ &$\checkmark$ &$\checkmark$, MO\\
\hline

\end{tabular}
\egroup
\end{tabular}
\caption{Comparison of results. LOO and MO denote linear optimization oracle and membership oracle respectively. $I$ denotes the length of any interval in adaptive regret. $\kappa=D/r$ denotes the ratio between enclosing and enclosed balls of the domain $\K$.}

\label{table:result_summary}
\end{center}
\end{table}

\subsection{Related work}

For a survey of the large body of work on online convex optimization, projection free methods and adaptive regret minimization see \cite{hazan2016introduction}.

The study of adaptive regret was initiated by \cite{hazan2009efficient}, which proves $\tilde{O}(\sqrt{T})$ and $O(\log^2 T)$ adaptive regret bounds for general convex and strongly-convex loss functions respectively. The bound was later improved to $\tilde{O}(\sqrt{I})$ for any sub-interval with length $I$ by \cite{daniely2015strongly} for general convex losses. Further improvements were made in  \cite{jun2017improved,cutkosky2020parameter,lu2022adaptive}. \cite{lu2022efficient} also considers more efficient algorithms using only $O(\log \log T)$ number of experts, instead of the typical $O(\log T)$ experts.

Projection-free methods date back to the Frank Wolfe algorithm for linear programming \cite{frank1956algorithm}. In the setting of online convex optimization,  \cite{hazan2012projection} gave a the first sublinear regret algorithm with $O(T^{\frac{3}{4}})$ regret based on the Frank Wolfe algorithm. Later, \cite{garber2015faster} improved the regret bound to the optimal $O(\sqrt{T})$ under the assumption that the decision set is a polytope. Further assuming the loss is smooth or strongly convex, it was shown that $O(T^{\frac{2}{3}})$ regret is attainable  \cite{hazan2020faster, kretzu2021revisiting, wan2021projection}. 

This paper considers a particularly efficient projection method which can be performed in logarithmic time (and so is essentially projection-free). Similar projection methods have been previously applied in the context of online boosting algorithms \cite{pmlr-v139-hazan21a,brukhim2021online, brukhim2020online}, and projection-free OCO with membership oracles \cite{mhammedi2022efficient,levy2019projection}. 

Recent work by \cite{garber2022new} was the first to consider projection-free adaptive regret minimization. They consider using a linear optimization oracle, and achieve a $O(T^{\frac{3}{4}})$ adaptive regret bound. The work of \cite{garber2022new} also considers using a separation oracle, and gets an $O(\sqrt{T})$ regret bound in this setting. We improve their bound to $\tilde{O}(\sqrt{I})$ in the general sense of projection-free. However, our result is based on membership oracle computations instead of linear optimization, which can be either more or less efficient depending on the underlying convex set (but still strictly more efficient than separation oracles). 

\paragraph{Projection free methods using a membership oracle.} The approach of using a membership oracle rather than linear optimization started in the work of \cite{levy2019projection}. This line of work was significantly extended in the important work of  \cite{mhammedi2022efficient}, who introduced the  use of the function $\gamma(x)$ as a barrier. The main idea behind the algorithm of \cite{mhammedi2022efficient} is to first use the gauge function to reduce the problem to optimization over the unit ball by the geometric assumptions on the domain, then apply any black-box OCO algorithm over the unit ball. 

Instead, our approach directly uses the gauge function as a regularization function in the classical OGD framework, yielding simpler and more natural algorithm and analysis. In addition, the OGD framework makes it suitable for the adaptive regret minimization framework, and in a later section we will show Algorithm \ref{alg:1} inherently exhibits $O(\sqrt{T})$ adaptive regret when $\partial \K$ is smooth.

%% file: 2-prelim.tex
\section{Preliminaries}
We consider the online convex optimization (OCO) problem. At each round $t$ the player $\A$ chooses $x_t\in \K$ where $\K\subset \reals^d$ is some convex domain, then the adversary reveals convex loss function $f_t(x)$ and player suffers loss $f_t(x_t)$. The goal is to minimize regret:
$$
\regret(\A)=\sum_{t=1}^T f_t(x_t)-\min_{x\in \K} \sum_{t=1}^T f_t(x) 
$$
which corresponds to the difference between the overall loss the player suffered and that of the best fixed point in hindsight.

\subsection{Adaptive regret}
It's often the case that in a changing environment we want to have regret guarantees not only globally but also locally.  \cite{hazan2009efficient} introduced the notion of adaptive regret to capture this intuition to control the worst-case regret among all sub-intervals of $[1,T]$.

$$
\text{Adaptive-Regret}(\A)=\max_{I=[s,t]}\left( \sum_{\tau=s}^t f_{\tau}(x_{\tau})-\min_{x\in \K} \sum_{\tau=s}^{t} f_{\tau}(x)\right) .
$$

\cite{daniely2015strongly} extended this notion by adding dependence on the length of sub-intervals, and provided an algorithm that achieves an $\tilde{O}(\sqrt{|I|})$ regret bound for all sub-intervals $I$. In particular, they define strongly adaptive regret as follows:
$$
\text{SA-Regret}(\A,k)=\max_{I=[s,t],t-s=k}\left( \sum_{\tau=s}^t f_{\tau}(x_{\tau})-\min_{x\in \K} \sum_{\tau=s}^{t} f_{\tau}(x)\right) .
$$
The common strategy for minimizing adaptive regret is to hedge over some algorithm instances initiated on different intervals.

\subsection{Optimization oracles}
Most OCO algorithms require a projection operation per round for updates, which computes the following for some norm $\|\cdot \|$.
$$
\Pi_{\K}(x)=\argmin_{y\in K} \|x-y\|
$$
However, computing the projection might be computationally expensive which motivates the study of optimization with projection-free methods, which make use of other (easier) optimization oracles instead.

The most commonly used optimization oracle for projection-free methods is the linear optimization oracle, which outputs 
$$\argmax_{x\in \K} v^{\top}x$$
given any vector $v$. It's known that linear optimization oracle is computationally equivalent (up to polynomial time) to separation oracle, which given any $x$ outputs \textbf{yes} if $x\in \K$, or a separating hyperplane between $x$ and $\K$.

In this paper we consider the membership oracle, that given an input $x\in \reals^d$, outputs \textbf{yes} if $x\in \K$ and \textbf{no} otherwise. This oracle is considered dual of the linear optimization oracle used by the Frank-Wolfe algorithm, and we shall see it admits better adaptive regret bounds at the cost of only logarithmically more calls to the oracle.
$$
\text{Oracle($x$)=} \left\{
\begin{array}{lcl}
\textbf{yes}  & & x\in \K\\
\textbf{no} & & x\notin \K
\end{array}\right.
$$

\subsection{Assumptions}
We make the following assumption on the loss $f_t$.
\begin{assumption}\label{basic assumption}
    The loss $f_t$ is convex, $G$-Lipschitz and non-negative. The domain $\K$ has diameter $D$. For simplicity, we assume $G,D\ge 1$.
\end{assumption}

We also assume the loss functions are defined over $\reals^d$, though we only optimize over $\K$.
\begin{assumption}\label{extend assumption}
    The loss $f_t$ is defined on $\reals^d$ (with the same convexity and Lipschitzness).
\end{assumption}
In fact, we can extend any convex Lipschitz function $f$ to $\reals^d$ while preserving the Lipschitz constant by defining 
$$
\hat{f}(x)=\min_{y\in \K} f(y)+G\|x-y\|_2
$$
for details see Theorem 1 in \cite{cobzas1978norm}. To make use of the membership oracle, we assume the domain $\K$ contains a ball centered at origin.
\begin{assumption}\label{ball}
    $\K$ contains $r \mathcal{B}_d$ for some constant $r>0$, where $\mathcal{B}_d$ is the unit ball in $\reals^d$.
\end{assumption}

Besides general convex loss functions, we will also consider strongly-convex loss functions, defined as follows:

\begin{definition}
A function $f(x)$ is $\lambda$-strongly-convex if for any $x,y\in \mathcal{K}$ the following holds:
$$
f(y)\ge f(x)+\nabla f(x)^{\top}(y-x)+\frac{\lambda}{2}\|x-y\|_2^2
$$
\end{definition}

%% file: 3-algorithm.tex
\section{Projection-free Algorithm via a Membership Oracle}
\begin{algorithm}
\caption{Projection-free algorithm via a membership oracle}
\label{alg:1}
\begin{algorithmic}[1]
\STATE Input: time horizon $T$, initialization $x_1, y_1 \in \K$, $\delta=\frac{1}{T^2}$. 
\FOR{$t = 1, \ldots, T$}
\STATE Play $x_t$ and suffer loss $f_t(x_t)$
\STATE Observe loss $f_t$, define $\hat{f}_t(x) = f_t(x) + 3G D (\gamma(x)-1)$
\STATE Compute $\tilde{\gamma}(y_t)$ and $\tilde{\nabla} \gamma(y_t)$, a $\delta$-approximation of $\gamma(y_t)$ and $\nabla \gamma(y_t)$ { (see Lemmas \ref{zaklemma1}, \ref{zaklemma2})}
\STATE Update $x_t$ via the OGD rule and project via the Minkowski regularization:
\[ y_{t+1} \leftarrow y_{t} - \eta_t  \tilde{\nabla} \hat{f}_{t}(y_t)=  y_{t} - \eta_t  (\nabla f_t(y_t)+2GD \tilde{\nabla} \gamma(y_t)).\]
\[ x_{t+1} \leftarrow   \frac{y_{t+1}}{\tilde{\gamma}(y_{t+1})} .\]
\ENDFOR
\end{algorithmic}
\end{algorithm}

Let $\K \subseteq \reals^d$, and contains $r \mathcal{B}_d$ for some constant $r > 0$ as in Assumption \ref{ball}. Define the Minkowski regularization for a convex set $\K \subseteq \reals^d$ as 
\begin{equation}
    \gamma(x) = \inf\Bigg\{ c \geq 1  \ : \  \frac{x}{c} \in \K \Bigg\}.
\end{equation}

The "projection operation" (which we call the Minkowski projection) w.r.t. $\gamma(x)$ is particularly simple to compute via $O(\log \frac{1}{\delta})$ membership oracle calls up to precision $\delta$, without the use of any projection operator. It is defined as $$ \Pi_\gamma(x) = {\frac{x}{\gamma(x)}}.$$

Our main algorithm \ref{alg:1} is simply a lazy version of OGD run on the original loss function with the Minkowski regularization, using the Minkowski projection after the gradient update. The algorithm uses $O(\log T)$ calls to the membership oracle to get $\frac{1}{T^2}$ approximate of $\gamma(x)$ and the gradient of $\gamma(x)$. We will introduce some important properties of the Minkowski regularization to establish the validity of using such approximation.

\subsection{Properties of the Minkowski regularization}
To execute Algorithm \ref{alg:1}, one needs to compute $\gamma(x)$ and $\nabla \gamma(x)$ to high accuracy using only the membership oracle. For the analysis henceforth, we require the following properties on the approximation of the regularization function and its gradient. 

\begin{lemma}\label{zaklemma1}
(\textbf{Approximating $\gamma(x)$}) 
Using $\log_2 \frac{2D}{\delta}$ number of calls to the membership oracle, we can compute $\Pi_{\gamma}(x)$ to $\delta$ accuracy. By $\log_2 \frac{2D^2}{r^2 \delta}$ number of calls to the membership oracle, we can compute $\gamma(x)$ to $\delta$ accuracy.
\end{lemma}

\begin{lemma}\label{zaklemma2}[Proposition 11 in \cite{mhammedi2022efficient}] (\textbf{Approximating $\nabla \gamma(x)$})
Using $O(d \log\frac{Dd}{r \delta})$ number of calls to the membership oracle, we can get an output denoted as $s$, such that $\E[s]$ is an $\delta$-approximate of the (sub)gradient of $\gamma(x)$, and $\E[\|s\|_2^2]\le \frac{2}{r^2}$.
\end{lemma} 

The methods to achieve the approximation in the above lemmas are intuitively simple: a binary line search is used to approximate $\gamma(x)$, based on which a random partial difference along different coordinates is used to approximate the (sub)gradient. We include the proof of Lemma \ref{zaklemma1} in appendix, and refer the readers to \cite{mhammedi2022efficient} or \cite{lee2018efficient} for the more technical Lemma \ref{zaklemma2}. 

As a result, if we choose $\delta=\frac{1}{T^2}$, then by the Lipschitzness of loss functions, the overall difference on the accumulated loss is only $o(1)$ which is negligible compared with regret bounds. The number of calls to the membership oracle is only $O(d\log T)$. 

Below we introduce two technical lemmas necessary for our analysis, whose proof can be found in the appendix. We first show the convexity and Lipschitzness of $\gamma(x)$ in the following lemma. 
\begin{lemma} \label{lemma:mink_bar_prop}
The Minkowski regularization $\gamma(x)$ is convex and $\frac{1}{r}$-Lipschitz.
\end{lemma}

By the convexity of $\gamma(x)$ and the fact that $\gamma(x)$ is defined on $\reals^d$, the well-known Alexandrov theorem implies that it has a second derivative almost everywhere, and without loss of generality we can consider only gradients of $\gamma(x)$ instead of subgradients. We also prove a useful property of $\hat{f}$ and the Minkowski projection.
\begin{lemma}\label{property}
The Minkowski projection only reduces the function value of $\hat{f}$ if $\gamma(x)$ can be computed exactly. When $\Pi_{\gamma}(x)$ is computed to $\frac{1}{T^2}$ precision, we have that $\hat{f}_t(x_t)\le \hat{f}_t(y_t)+\frac{Gr+3GD}{rT^2}$
\end{lemma}

This lemma is powerful in that it allows us to replace the harder problem of learning $x_t$ over a constrained set $\K$ (which needs projection) to learning $y_t$ over $\reals^d$, by replacing the original loss $f_t$ with $\hat{f}_t$. The only thing left to check is whether $\hat{f}_t$ preserves the nice properties of $f_t$. It follows from definitions that for general convex or strongly-convex Lipschitz $f_t$, such a modification of the loss is 'for free'. Unfortunately, for exp-concave loss functions $\hat{f}_t$ will lose the exp-concavity.

%% file: analysis.tex
\section{Regret Guarantees and Analysis of Algorithm \ref{alg:1}}
In this section we prove regret bounds of Algorithm \ref{alg:1} in the non-adaptive setting, in order to build adaptive regret guarantees upon it in the next section. We start with regret guarantee for general convex loss functions.

\begin{theorem}\label{theorem: main}
Under Assumptions \ref{basic assumption}, \ref{extend assumption}, \ref{ball}, Algorithm \ref{alg:1} with $\eta_t=\frac{r}{\sqrt{T}G}$ can be implemented with $O(d\log T)$ calls to the membership oracle per iteration and has regret bounded by 
$$  \E[\sum_t f_t(x_t) - \min_{x \in \K} f_t(x)]  = O( \frac{D^2 G \sqrt{T}}{r} ) $$
\end{theorem}

\begin{proof}
Denote $\eta_t=\eta$. By Lemma \ref{property} we know that $\hat{f}_t(x_t)\le \hat{f}_t(y_t)+\frac{Gr+3GD}{rT^2}$. Moreover, by Lemma \ref{lemma:mink_bar_prop}, and assumption \ref{basic assumption} that $f_t$ is convex and $G$-Lipschitz, we get that $\hat{f}_t$ is convex and $G$-Lipschitz as well.  
Note that for any $x \in \K$, $\hat{f}_t(x) = f_t(x)$. 
Thus, we get,
\begin{align*}
\E[\sum_t f_t(x_t) - f_t(x^*)] & = \E[\sum_t \hat{f}_t(x_t) - \hat{f}_t(x^*)]  \\
& \leq  \E[\sum_t \hat{f}_t(y_t) - \hat{f}_t(x^*)] +\frac{Gr+3GD}{rT} \tag{by Lemma \ref{property}}.
 \end{align*}

 Next, we bound the above by showing that,
 $$
 \E[\sum_t \hat{f}_t(y_t) - \hat{f}_t(x^*)] =  O( D \tilde{G} \sqrt{T} )=O( \frac{D^2 G\sqrt{T}}{r} ),
 $$
where $\tilde{G}$ is the lipschitz constant of $\hat{f}_t$. To show that, we denote $\tilde{\nabla}_t \doteq \nabla \hat{f}_t(y_t)$, $\triangle_t \doteq \tilde{\nabla}\gamma(y_t)-\nabla \gamma(y_t)$ and observe that:

\begin{align*}
\|y_{t+1} - x^\star\|_2^2 = \|y_t - x^\star\|_2^2 &+ \eta^2 \|\tilde{\nabla}_t\|_2^2 - 2 \eta \tilde{\nabla}_t^\top (y_t - x^\star)\\
&+3GD\eta\triangle_t^{\top}(y_t-x^*-\eta \tilde{\nabla}_t)+9G^2D^2\eta^2\|\triangle_t\|_2^2,
\end{align*}
and by re-arranging terms and taking expectation we get,
\begin{align*}
    \E[2 \tilde{\nabla}_t^\top (y_t - x^\star)] &\le \dfrac{ \E[\|y_t - x^\star\|_2^2] - \E[\|y_{t+1} - x^\star\|_2^2]}{\eta} + \eta \tilde{G}^2+ \frac{9GD^2 \eta}{T^2}+\frac{36G^2D^2\eta^2}{r^2}.
\end{align*}
Telescoping and the fact that $y_1 \in \mathcal{K}$ yields the stated bound since by convexity:

\begin{align*}
\E[\sum_t \hat{f}_t(y_t) - \hat{f}_t(x^*)] &\leq \E[\sum_t \tilde{\nabla}_t^\top (y_t - x^\star)] \\
&\leq \dfrac{\|y_1 - x^\star\|_2^2}{2\eta} + \dfrac{\eta T \tilde{G}^2}{2}+
\frac{9GD^2 \eta}{2T}+\frac{18G^2D^2\eta^2T}{r^2}\end{align*}

We find that $\tilde{G}$ is upper bounded by the sum of the gradient of $f_t$ and that of $3GD \gamma(x)$. Recall that $\gamma$ is $1/r$-Lipschitz (by Lemma \ref{lemma:mink_bar_prop}), plugging in the appropriate $\eta=\frac{r}{\sqrt{T}G}$ gives the stated bound.
\end{proof}

Similarly, we can get logarithmic regret guarantee for strongly-convex loss functions.

\begin{theorem}\label{theorem: strong}
Under Assumptions \ref{basic assumption}, \ref{extend assumption}, \ref{ball} and further assume the loss functions are $\lambda$-strongly convex, Algorithm \ref{alg:1} with $\eta_t=\frac{1}{\lambda t}$ can be implemented with $O(d\log T)$ calls to the membership oracle per iteration and has regret bounded by 
$$  \E[\sum_t f_t(x_t) - \min_{x \in \K} f_t(x)]  = O( \frac{G^2D^2 \log T}{r^2\lambda} ) $$
\end{theorem}

\begin{proof}
By Lemma \ref{property} we know that $\hat{f}_t(x_t)\le \hat{f}_t(y_t)+\frac{Gr+3GD}{rT^2}$. In addition, $\hat{f}_t$ is strongly-convex and Lipschitz by Lemma \ref{lemma:mink_bar_prop}. We know that
$$
\E[\sum_t f_t(x_t) - f_t(x^*)] \le \E[\sum_t \hat{f}_t(y_t) - \hat{f}_t(x^*)]+\frac{Gr+3GD}{rT}
$$
Notice that $\hat{f}_t$ is still $\lambda$-strongly-convex. Apply the definition of strong convexity to $y_t, x^*$, we get
$$
\hat{f}_t(y_t) - \hat{f}_t(x^*)\le \tilde{\nabla}_t^{\top}(y_t-x^*)-\frac{\lambda}{2}\|y_t-x^*\|_2^2
$$
We proceed to upper bound $\tilde{\nabla}_t^{\top}(y_t-x^*)$. Using the update rule of $y_t$, we have
\begin{align*}
\|y_{t+1} - x^\star\|_2^2 = \|y_t - x^\star\|_2^2 &+ \eta_t^2 \|\tilde{\nabla}_t\|_2^2 - 2 \eta_t \tilde{\nabla}_t^\top (y_t - x^\star)\\
&+ 3GD\eta_t\triangle_t^{\top}(y_t-x^*-\eta_t \tilde{\nabla}_t)+9G^2D^2\eta^2\|\triangle_t\|_2^2
\end{align*}
and so, 
\begin{align*}
    \E[2 \tilde{\nabla}_t^\top (y_t - x^\star)] &\le \dfrac{ \E[\|y_t - x^\star\|_2^2] - \E[\|y_{t+1} - x^\star\|_2^2]}{\eta_t} + \eta_t \tilde{G}^2+\frac{9GD^2 \eta_t}{T^2}+\frac{36G^2D^2\eta_t^2}{r^2}.
\end{align*}

hence the regret $\E[\sum_t f_t(x_t) - f_t(x^*)]$ can be upper bounded as follows, 
\begin{align*}
\E[\sum_t f_t(x_t) - f_t(x^*)] &\le  \sum_t \E[\tilde{\nabla}_t^{\top}(y_t-x^*)-\frac{\lambda}{2}\|y_t-x^*\|_2^2]+\frac{Gr+3GD}{rT}\\
&\le \sum_t \frac{1}{2} \E[\|y_t-x^*\|_2^2] (\frac{1}{\eta_t}-\frac{1}{\eta_{t-1}}-\lambda) +(\frac{\tilde{G}^2}{2}+\frac{9GD^2}{2T^2})\sum_t \eta_t\\
& \qquad \ \ +\frac{Gr+3GD}{rT}+\frac{18G^2D^2}{r^2}\sum_t \eta_t^2\\
&=0+(\frac{\tilde{G}^2}{2\lambda}+\frac{9GD^2}{2\lambda T^2}) \sum_t \frac{1}{t}+\frac{Gr+3GD}{rT}+\frac{18G^2D^2}{r^2\lambda^2}\sum_t \frac{1}{t^2}\\
&\le \frac{G^2D^2}{2r^2\lambda}(1+\log T)+\frac{3\pi^2 G^2D^2}{r^2\lambda^2}+o(1)\\
&=O( \frac{G^2D^2 \log T}{r^2\lambda} ).
\end{align*}
\end{proof}

%% file: 4-adaptive.tex
\section{Projection Free Methods with Adaptive regret Guarantees}

In this section we consider projection-free adaptive regret algorithms based on Algorithm \ref{alg:1}. We show how to achieve adaptive regret using membership oracles for both general convex and strongly-convex loss functions, by using Algorithm 1 as the base algorithm in adaptive regret meta algorithms (details in appendix).

For strongly-convex loss, we combine Algorithm \ref{alg:1} with the FLH algorithm  to get $O(\log^2 T)$ adaptive regret bound, using $O(d\log^2 T)$ calls to the membership oracle. Next, we show that combining Algorithm \ref{alg:1} with the EFLH algorithm  \cite{lu2022efficient} yields an $\tilde{O}(I^{\frac{1}{2}+\epsilon})$ strongly adaptive regret bound using only $O(d\log T\times\log \log T/\epsilon)$ calls to the membership oracle per round. We start with the case of strongly convex loss case:

\begin{theorem}\label{ada1}
Under Assumptions \ref{basic assumption}, \ref{extend assumption}, \ref{ball}, and further assume the loss functions are $\lambda$-strongly convex, with $O(d\log^2 T)$ calls to the membership oracle Algorithm \ref{alg: adaptive strongly convex} achieves expected adaptive regret
$$ O( \frac{G^2D^2 \log^2 T}{r^2\lambda})$$
\end{theorem}

For general convex loss, we show how to get an improved $\tilde{O}(I^{\frac{1}{2}+\epsilon})$ strongly adaptive regret bound when $I$ is the length of any sub-interval $[s,t]$, using only $O(d\log \log T/\epsilon)$ calls to the membership oracle per round for any $\epsilon>0$. The main idea is to use Algorithm \ref{alg:1} as a black-box, and apply it under the framework of \cite{lu2022efficient}.

\begin{theorem}\label{ada2}
Under Assumptions \ref{basic assumption}, \ref{extend assumption}, \ref{ball}. By using Algorithm \ref{alg:1} as the black-box algorithm in the EFLH algorithm of \cite{lu2022efficient}, Algorithm \ref{alg: adaptive convex} achieves the following expected adaptive regret bound for all intervals $[s,t]$ with $O(d\log T \times \log \log T/\epsilon)$ calls to the membership oracle per round.
$$ \E[\sum_{i=s}^t f_i(x_i)-\min_{x\in \K} \sum_{i=s}^t f_i(x)]=O( \frac{D^2 G \sqrt{\log T} (t-s)^{\frac{1}{2}+\epsilon}}{r} ) $$
\end{theorem}

Notice that by choosing $\epsilon=\frac{1}{\log T}$, we have that
$$
I^{\frac{1}{2}+\epsilon}\le \sqrt{I} T^{\frac{1}{\log T}}=O(\sqrt{I})
$$
which recovers the optimal $\tilde{O}(\sqrt{I})$ strongly adaptive regret bound, but using $O(\log T)$ number of experts instead. The proofs are direct black-box reductions by replacing the expert OCO algorithms in adaptive regret algorithms by our algorithm, thus we leave the (simple) proofs to appendix. 

\begin{remark}
Adaptive regret bounds can be derived similarly from \cite{mhammedi2022efficient}, by choosing the subroutine in the reduction of \cite{mhammedi2022efficient} to have an adaptive regret over the ball.
\end{remark}

%% file: smooth.tex
\section{A Better Gradient Approximation of the Minkowski Regularization}
The approximation of $\nabla \gamma(x)$ in Lemma \ref{zaklemma2} is limited in two aspects: it doesn't have a uniform control of the approximation error like in Lemma \ref{zaklemma1}, for which we can only achieve expected regret guarantees. Secondly, the method used by Lemma \ref{zaklemma2} is designed for general functions and does not exploit the special structure of $\gamma(x)$.

In this section we build a better method to obtain uniform control of the approximation error for estimating $\nabla \gamma(x)$ when $\partial \K$ is smooth. We later show how even non-smooth sets, such as polytopes, can be smoothed, and the techniques hereby still apply. Denote by $v(x)$ the normal vector to its tangent plane for any $x\in \partial \K$. We start by stating that Algorithm \ref{alg:1} without modification possesses an  adaptive regret guarantee for smooth sets.

\begin{theorem}\label{smooth theorem}
Under Assumptions \ref{basic assumption}, \ref{extend assumption}, \ref{ball}.
If $\partial \K$ is smooth and for any $x\in \partial \K$ we have $\frac{v(x)}{x^{\top} v(x)}
$ is Lipschitz, Algorithm \ref{alg:1} with $\eta_t=\frac{D}{\sqrt{T}G}$ can be implemented with $O(d\log T)$ calls to the membership oracle per iteration and has adaptive regret bounded by 
$$ \max_{1\le s\le t\le T} \left \{\sum_{i=s}^t f_i(x_i)-\min_{x\in \K} \sum_{i=s}^t f_i(x)\right \}=O( \frac{D^3 G \sqrt{T}}{r^2} ) $$
\end{theorem}

The core technique is to show that the smoothness of $\partial \K$ implies the smoothness of $\gamma(x)$, by which we can obtain a uniform upper bound on the approximation error formalized in the following lemma. 

\begin{lemma}\label{smooth gradient}
Assume $\partial \K$ is a smooth manifold, such that for any $x\in \partial \K$, its normal vector $v(x)$ to its tangent plane is unique. For any $x\notin \K$, $\nabla \gamma(x)$ is along the same direction of $v(\Pi_{\gamma}(x))$. In addition, the gradient $\nabla \gamma(x)$ is
$$
\frac{v(\Pi_{\gamma}(x))}{\Pi_{\gamma}(x)^{\top} v(\Pi_{\gamma}(x))}
$$
\end{lemma}

The intuition is that the Minkowski functional defines a norm, under which $\partial \K$ becomes the unit sphere, then the gradient of norm on the sphere should be the normal vector to the tangent plane. We are able to build a better estimation to $\nabla \gamma(x)$ based on this expression.
\begin{lemma}\label{smooth}
Assume the set $\K$ satisfies that for any $x\in \partial \K$ $\frac{v(x)}{x^{\top} v(x)}
$ is $\beta$-Lipschitz, then using $O(d \log\frac{Dd}{r \delta})$ number of calls to the membership oracle, we can get an $\delta$-approximate of the gradient of $\gamma(x)$. 
\end{lemma}
\begin{remark}
If $\K$ is a polytope defined by $\{x| h(x)\le 0\}$ where $h(x)=\max_{i=1}^m (x^{\top} \alpha_i+b_i)$, we can smooth the set to get a similar result. Details can be found in appendix.
\end{remark}

\begin{algorithm}
\caption{Estimating $\nabla \gamma(x)$ with smooth $\partial \K$}
\label{alg:smooth}
\begin{algorithmic}[1]
\STATE Input: $\lambda=\frac{1}{\sqrt{d} T^{2.5}}$, $\delta=\frac{1}{d T^{5}}$ and smoothness factor $\beta$ of $\partial \K$.
\STATE Estimate $\gamma(x)$ to $\delta$ accuracy.
\FOR{$i = 1, \ldots, d$}
\STATE Estimate $\gamma(x+\lambda e_i)$ to $\delta$ accuracy.
\STATE Estimate $\nabla \gamma(x)_i$ by $\frac{\gamma(x+\lambda e_i)-\gamma(x)}{\lambda}$
\ENDFOR
\end{algorithmic}
\end{algorithm}

Lemma \ref{smooth} also enables us to remove the expectation on the regret bound in previous theorems, because estimations to both $\gamma(x)$ and $\nabla \gamma(x)$ can be implemented deterministically, and we now have a uniform control on the estimation error of $\nabla \gamma(x)$.

When the set $\K$ is defined by $\{x| h(x)\le 0\}$ for some smooth function $h(x)$, if $h(x)$ itself is smooth it also satisfies the above assumption where $\frac{\nabla h(x)}{x^{\top}\nabla h(x)}$ is $\frac{D\beta}{r}$-Lipschitz on $\partial \K$. 

\begin{proof}[Proof of Theorem \ref{smooth theorem} ]
The proof is straightforward by noticing Algorithm \ref{alg:1} is an active algorithm. In fact, using the same analysis as in Theorem \ref{theorem: main}, the telescoping gives
$$
\sum_{i=s}^t \hat{f}_i(y_i) - \hat{f}_i(x^*) \leq \dfrac{\|y_s - x^\star\|_2^2}{2\eta} + \dfrac{\eta (t-s) \hat{G}^2}{2}+\frac{9GD^2 \eta}{T}
$$
The only question here is how to control the norm of $y_s$. We argue here that for any $t$ 
$$\|y_t\|_2\le \frac{r}{3}(2+\frac{3D}{r})^2+3D+2r=O(\frac{D^2}{r})$$
Use the definition of update rule, we get
$$\|y_{t+1}\|_2^2-\|y_t\|_2^2=\frac{r^2}{TG^2}\|\tilde{\nabla} \hat{f}_t\|_2^2-2y_t^{\top} \frac{r}{\sqrt{T}G}\tilde{\nabla} \hat{f}_t$$
Notice that when $y_t\notin \K$, $\frac{y_t^{\top}}{\|y_t\|_2} \nabla \gamma(y_t)= \frac{1}{\|\gamma(y_t)\|_2}$ which is in $[\frac{1}{D}, \frac{1}{r}]$, combine this with the fact that $f_t$ is $G$-Lipschitz, we have that 
$$
y_t^{\top} \tilde{\nabla} \hat{f}_t\ge y_t^{\top} \nabla \hat{f}_t-\frac{3GD}{rT^2}\|y_t\|_2\ge \|y_t\|_2 (\frac{3GD}{\|\gamma(y_t)\|_2}-G-\frac{3GD}{rT^2})\ge 1.5G\|y_t\|_2
$$
We also need to upper bound $\|\tilde{\nabla} \hat{f}_t\|_2^2$, which is $(G+\frac{3GD}{r}+\frac{3GD}{T^2})^2$.

As a result, when 
$$
\|y_t\|_2\ge \frac{r}{3}(2+\frac{3D}{r})^2
$$
we have that
$$
2y_t^{\top} \frac{r}{\sqrt{T}G}\tilde{\nabla} \hat{f}_t\ge \frac{r^2}{\sqrt{T}}(2+\frac{3D}{r})^2\ge\frac{r^2}{TG^2}\|\tilde{\nabla} \hat{f}_t\|_2^2
$$
and further $\|y_{t+1}\|_2\le \|y_t\|_2$.

Besides, the distance that $y_t$ can move in a single step is upper bounded by $(\frac{3GD}{r}+\frac{3GD}{T^2}+G)\frac{r}{\sqrt{T}G}\le 3D+2r$, which concludes the proof of the argument. Now, the choice of $\eta=\frac{D}{\sqrt{T}G}$ which is independent of $s,t$ gives an $O( \frac{D^3 G \sqrt{T}}{r^2} ) $ regret bound over all intervals $[s,t]$ simultaneously because $t-s\le T$.
\end{proof}

Simply using the previous estimation of $\nabla \gamma(x)$ (Lemma \ref{zaklemma2}) doesn't work, because its variance causes an unpleasant martingale representation of $y_t$, making it hard to bound its norm.

%% file: conclusion.tex
\section{Conclusion}
In this paper, we consider the problem of online convex optimization with membership oracles. We propose a simple lazy OGD algorithm that achieves $O(\sqrt{T})$ and $O(\log T)$ regret bounds for general convex and strongly-convex loss functions respectively, using $O(d\log T)$ number of calls to the membership oracle. We further utilize the active nature of our algorithm in adaptive regret minimization, achieving better bounds than known algorithms using linear optimization oracles. It remains open to try to match these regret bounds using linear optimization, or alternatively fewer membership oracle queries.

\section*{Acknowledgements}

The authors thank Dan Garber and Zak Mhammedi for helpful suggestions and typo corrections in an earlier version of this manuscript.  
Zhou Lu, Nataly Brukhim and Elad Hazan are supported by NSF grant number 2134040.

%% file: A-appendix.tex
\section{Adaptive Regret Meta Algorithms}
We include here the algorithm boxes of the FLH and EFLH algorithms. Both Algorithm \ref{alg: adaptive strongly convex} and Algorithm \ref{alg: adaptive convex} are meta expert algorithms, with some black-box base algorithms run on specific intervals.

\begin{algorithm}[t]
\caption{Adaptive regret for strongly-convex loss without projection}
\label{alg: adaptive strongly convex}
\begin{algorithmic}[1]
\STATE Input: OCO algorithm $\mA$, active expert set $S_t$, horizon $T$, $\alpha=\frac{\lambda}{G^2}$ and constant $\epsilon>0$.
\STATE Set $\mA_j$ to be Algorithm \ref{alg:1} with $\eta_t=\frac{1}{\lambda t}$. \STATE Pruning rule: the horizon of $\mA_j$ is $2^{k+2}+1$ if $j=r 2^k$ where $r$ is an odd number.
\STATE Initialize: $S_1=\{1\}$, $p_1^1=1$.
\FOR{$t = 1, \ldots, T$}
\STATE Set $\forall j\in S_t$, $x_t^j$ to be the prediction of $\mA_j$.
\STATE Play $x_t=\sum_{j\in S_t} p_t^j x_t^j$.
\STATE Perform multiplicative weight update. For $j\in S_t$
$$
\tilde{p}_{t+1}^j=\frac{p_t^j e^{-\alpha f_t(x_t^j)}}{\sum_{i\in S_t} p_t^i e^{-\alpha f_t(x_t^i)}}
$$
\STATE Prune $S_t$ and add $\{t+1\}$ to get $S_{t+1}$. Initialize $\tilde{p}_{t+1}^{t+1}=\frac{1}{t}$, then
$$
\forall j\in S_{t+1}, p_{t+1}^j=\frac{\tilde{p}_{t+1}^j}{\sum_{i\in S_{t+1}} \tilde{p}_{t+1}^i}
$$
\ENDFOR
\end{algorithmic}
\end{algorithm}

\begin{algorithm}[t]
\caption{Strongly adaptive regret for general convex loss without projection}
\label{alg: adaptive convex}
\begin{algorithmic}[1]
\STATE Input: OCO algorithm $\mA$, active expert set $S_t$, horizon $T$ and constant $\epsilon>0$.
\STATE Let $\mA_{t,k}$ be an instance of $\mA$ initialized at $t$ with lifespan $4 l_k=4\lfloor 2^{(1+\epsilon)^k}/2 \rfloor+4$, for $2^{(1+\epsilon)^k}/2\le T$. Here $\mA_{t,k}$ is set to be Algorithm \ref{alg:1} with horizon $4l_k$ and $\eta=\frac{r}{2\sqrt{l_k}G}$.
\STATE Initialize: $S_1=\{(1,1),(1,2),...\}$, $w_1^{(1,k)}=\min \left\{\frac{1}{2},\sqrt{\frac{\log T}{l_k}}\right\}$.
\FOR{$t = 1, \ldots, T$}
\STATE Let $W_t=\sum_{(j,k)\in S_t} w_t^{(j,k)}$.
\STATE Play $x_t=\sum_{(j,k)\in S_t} \frac{w_t^{(j,k)}}{W_t} x_t^{(j,k)}$, where $x_t^{(j,k)}$ is the prediction of $\mA_{(j,k)}$.
\STATE Perform multiplicative weight update to get $w_{t+1}$. For $(j,k)\in S_t$
$$
w_{t+1}^{(j,k)}=w_t^{(j,k)} \left(1+\min \left\{\frac{1}{2},\sqrt{\frac{\log T}{l_k}}\right\} (f_t(x_t)-f_t(x_t^{(j,k)})) \right)
$$
\STATE Update $S_t$ according to the pruning rule. Initialize 
$$
w_{t+1}^{(t+1,k)}=\min \left\{\frac{1}{2},\sqrt{\frac{\log T}{l_k}}\right\}
$$
if $(t+1,k)$ is added to $S_{t+1}$ (when $l_k | t-1$).
\ENDFOR
\end{algorithmic}
\end{algorithm}

\section{Omitted Proofs}
\subsection{Proof of Lemma \ref{lemma:mink_bar_prop}}
\begin{proof}
For any two points $x_1,x_2\in \reals^d$, we let $\gamma_1=\gamma(x_1),\gamma_2=\gamma(x_2)$. For any $\lambda\in (0,1)$, let $z=\lambda x_1+(1-\lambda)x_2$. To establish convexity, we would like to prove that
\begin{equation}\label{eq:conv_gamma}
\gamma(z) \le \lambda \gamma_1+(1-\lambda)\gamma_2.    
\end{equation}
Since $\gamma(z)$ is the minimal value (greater than $1$) such that $z/\gamma(z) \in \K$, we get that proving that Equation \eqref{eq:conv_gamma} holds is equivalent to showing that,
\begin{equation}
    \frac{z}{\gamma(z)} = \frac{\lambda x_1 + (1-\lambda) x_2}{\lambda \gamma_1 + (1-\lambda) \gamma_2} \in \K.
\end{equation}
First, notice that the term $z/\gamma(z)$ can be rewritten as follows, 
\begin{equation}
    \frac{z}{\gamma(z)} = \frac{x_1}{\gamma_1}\cdot\Bigg(\frac{1}{1 + (\frac{1}{\lambda}-1)\frac{\gamma_2}{\gamma_1}}\Bigg) + \frac{x_2}{\gamma_2}\cdot\Bigg(
\frac{1}{1 + \frac{\lambda}{1-\lambda}\frac{\gamma_1}{\gamma_2}}\Bigg).
\end{equation} 
Then, it is easy to verify that for any $\lambda \in (0,1)$, and any $\gamma_1,\gamma_2 \ge 1$, both terms in the brackets are non-negative and sum to $1$. Moreover, since by definition we have that $\frac{x_1}{\gamma_1},\frac{x_2}{\gamma_2}\in \K$, we get that $\frac{z}{\gamma(z)}$ is a convex combination of elements in $\K$, and thus must be in $\K$ as well. This concludes the proof that $\gamma(x)$ is convex. \\

Next, to prove Lipschitzness, notice that when $x_1,x_2\in \K$ the argument is trivially true. Otherwise, assume without loss of generality that $x_2\notin \K$ (and any $x_1$). Decompose $x_2$ as,
\begin{equation}
x_2=x_1+r\frac{x_2-x_1}{\|x_2-x_1\|_2} \times \frac{\|x_2-x_1\|_2}{r}.    
\end{equation}
Then, we define $\alpha = \frac{r}{\|x_2-x_1\|_2+r\gamma_1}$, where $r$ is the radius of a ball $r \mathcal{B}_d$ that is  contained in $\K$ (recall assumption \ref{basic assumption}). Using out decomposition of $x_2$ above, we get that,
\begin{equation}
\alpha \cdot x_2= (\gamma_1\alpha) \cdot \frac{x_1}{\gamma_1}+(1-\gamma_1\alpha)\cdot r\frac{x_2-x_1}{\|x_2-x_1\|_2}.
\end{equation}
Observe that since $r\frac{x_2-x_1}{\|x_2-x_1\|_2} \in r \mathcal{B}_d$ then is it also contained in $\K$, by assumption \ref{basic assumption}. In addition, by definition we have $\frac{x_1}{\gamma_1} \in \K$. 
Thus, the term $\alpha \cdot x_2$ can be re-written as a convex combination of elements in $\K$, and is therefore contained in $\K$ as well. Since $\gamma_2$ is the minimal scalar such that division of $x_2$ by it is in $\K$, we get that $\gamma_2 \le 1/\alpha$. That is,   
\begin{equation}
    |\gamma_2 - \gamma_1|  \le \frac{1}{r} \|x_2-x_1\|_2,
\end{equation}
which concludes the proof. 
\end{proof}

\subsection{Proof of Lemma \ref{property}}

\begin{proof}
In the exact case, denote $\gamma_t:= \gamma(y_{t-1})$. Then, we have that 
\begin{align*}
     \hat{f}_t(x_t) -  \hat{f}_t(y_t) & = f_t(x_t) - f_t(y_t) + 3G D  (\gamma(x_t) -1) - 3G D (\gamma_t - 1)  \tag{definition of $\hat{f}_t$} \\
&=  f_t(x_t) - f_t(y_t)  - 3G D(\gamma_t-1)\tag{$\gamma(x_t)=1$} \\
&\leq  G \|y_t - x_t\|_2  - 3G D (\gamma_t-1) \tag{$f_t$ is $G$-Lipschitz} \\
& = G  (\gamma_t - 1) \|x_t\|_2- 3G D (\gamma_t-1) \tag{$y_t = \gamma_t x_t$} \\
& = G  (\gamma_t - 1) \big(\|x_t\|_2 - 3D \big)  \\
&\le  0 \tag{$\|x_t\|_2 \le 3D, \gamma_t \ge 1$}
\end{align*}
We notice that a similar argument also holds for the approximate version of $\gamma_t$: just replace $x_t=\frac{y_t}{\gamma_t}$ by $\tilde{x}_t=\frac{y_t}{\tilde{\gamma}_t}$ and use the fact that $\|x_t-\tilde{x}_t \|_2\le \frac{1}{T^2}$, which gives $\hat{f}_t(\tilde{x}_t) -  \hat{f}_t(y_t)\le \frac{Gr+3GD}{rT^2}$. This supports our previous argument that the overall difference on the accumulated loss is only $O(1)$ and is negligible.
\end{proof}

\subsection{Proof of Lemma \ref{zaklemma1}}
\begin{proof}
Without loss of generality we only consider the case $x\notin \K$. We use the simple binary search method: the first call to the oracle gives $x\notin \K$, and we save two points $x_{in}$ and $x_{out}$, initialized as $0$ and $x$. For the $t+1$-th call to the oracle, we query $\frac{1}{2}(x_{in}+x_{out})$ and set it as $x_{in}$ or $x_{out}$ according to whether it's in $\K$ or not.

It follows that after the $t$-th call to the oracle, $\|x_{in}-x_{out}\|_2\le \frac{2D}{2^t}$ which implies that $x_{in}$ is an $\frac{2D}{2^t}$-approximation to $\Pi_{\gamma}(x)$. In other words, one can use $\log_2 (\frac{2D}{\epsilon})$ calls to the membership oracle to get an $\epsilon$-approximation to $\frac{x}{\gamma(x)}$. To put it more general, one can use $\log_2 (\frac{2D^2}{r^2 \epsilon})$ calls to the membership oracle to get an $\epsilon$-approximation to $\gamma(x)$.
\end{proof}

\subsection{Proofs of Theorem \ref{ada1} and Theorem \ref{ada2}}
\begin{proof}
Two theorems from previous works are introduced first, in which we will use our Algorithm \ref{alg:1} as a black-box to obtain the desired regret bounds.

\begin{theorem}[Theorem 3.1 in \cite{hazan2009efficient}]
If all loss functions are $\alpha$-exp concave then the FLH algorithm attains adaptive regret of $O(\frac{\log^2 T}{\alpha})$. The number of experts per iteration is $O(n^3 \log T )$.
\end{theorem}

\begin{theorem}[Theorem 6 in \cite{lu2022efficient}]
Given an OCO algorithm $\mA$ with regret bound $cGD  \sqrt{T}$ for some constant $c\ge 1$, the adaptive regret of EFLH is bounded by $40 cGD \sqrt{\log T} |I|^{\frac{1+\epsilon}{2}}$ for any interval $I\subset [1,T]$. The number of experts per round is $O(\log \log T/\epsilon)$.
\end{theorem}

We first define the exp-concavity of loss function.

\begin{definition}
A function $f(x)$ is $\alpha$-exp-concave if $e^{-\alpha f(x)}$ is convex.
\end{definition}

A notable fact is that under Assumption \ref{basic assumption}, any $\lambda$-strongly-convex function is also $\frac{\lambda}{G^2}$-exp-concave \cite{hazan2016introduction}.

For Theorem \ref{ada1}, we use the fact that the loss functions are also $\frac{\lambda}{G^2}$-exp-concave. By Theorem 3.1 in \cite{hazan2009efficient} and Theorem \ref{theorem: strong}, if we treat our Algorithm \ref{alg:1} as the experts in FLH, the total regret is bounded by $O(\frac{G^2D^2 \log^2 T}{r^2\lambda}+\frac{G^2 \log^2 T}{\lambda})$. Noticing $r=O( D)$, we reach our conclusion.

Theorem \ref{ada2} is a direct consequence of Theorem \ref{theorem: main} and Theorem 6 from \cite{lu2022efficient} by treating our Algorithm \ref{alg:1} as the experts in EFLH. Notice that the EFLH is a multiplicative weight algorithm over experts (instances of Algorithm \ref{alg:1}) and doesn't require extra projections.
\end{proof}

\subsection{Proof of Lemma \ref{smooth gradient}} 
\begin{proof}
To see this, we compare the values of $\gamma(x)$ and $\gamma(x+\epsilon v)$ given a small constant $\epsilon$ and a unit vector $v$, and see which direction increases the function value the most. We have that 
$$
\Pi_{\gamma}(x+\epsilon v)= \Pi_{\gamma}(x)+\epsilon \frac{\|\Pi_{\gamma}(x)\|_2}{\|x\|_2} (v-v^{\top}v(\Pi_{\gamma}(x))v(\Pi_{\gamma}(x))) +O(\epsilon^2)
$$
due to the fact that $\partial \K$ is smooth and at a neighborhood of $x$, its tangent plane is a good linear approximation of the boundary. As a result
$$
\gamma(x+\epsilon v)=\frac{\|x+\epsilon v\|_2}{\|\Pi_{\gamma}(x)+\epsilon \frac{\|\Pi_{\gamma}(x)\|_2}{\|x\|_2} (v-v^{\top}v(\Pi_{\gamma}(x))v(\Pi_{\gamma}(x))) +O(\epsilon^2)\|_2}
$$
We are actually dealing with the $2$-dimensional plane spanned by $v$ and $v(\Pi_{\gamma}(x))$, therefore without loss of generality we can do a coordinate change to only consider this $2$-dimensional plane and use $x=(\|x\|_2,0), \Pi_{\gamma}(x)=(\|\Pi_{\gamma}(x)\|_2,0)$, and parameterize $v$ and $v(\Pi_{\gamma}(x))$ by $\theta_1, \theta_2$. After standard calculation and simplification we have that 
$$
\gamma(x+\epsilon v)^2=\frac{\|x\|_2^2+2\epsilon\|x\|_2\cos \theta_1}{\|\Pi_{\gamma}(x)\|_2^2+2\epsilon \frac{\|\Pi_{\gamma}(x)\|_2^2}{\|x\|_2}\sin \theta_2 (\sin \theta_2 \cos \theta_1-\cos \theta_2 \sin \theta_1) }
$$
Maximizing it is equivalent to maximizing
$$
\frac{\|x\|_2+2\epsilon \cos \theta_1}{\|x\|_2+2\epsilon \sin \theta_2 (\sin \theta_2 \cos \theta_1-\cos \theta_2 \sin \theta_1) }
$$
which is further equivalent to maximizing
$$
\sin \theta_2 \sin \theta_1+\cos \theta_2 \cos \theta_1
$$
by ignoring second-order terms. It is maximized when taking $\theta_1=\theta_2$, or in other words, $\nabla \gamma(x)$ is along the same direction of $v(\Pi_{\gamma}(x))$. An easier way to think of it is by slightly modifying the definition of $\gamma(x)$ to be the original Minkowski functional
$$
\inf \{\gamma>0 : \frac{x}{\gamma}\in \K\}
$$
then $\gamma(x)$ defines a norm and $\partial \K$ is exactly the solution set of $\gamma(x)=1$. Then the gradient of $\gamma(x)$ is the normal vector to this norm-1 level set.

We haven't determine the magnitude of the gradients yet. From Lemma 6 in \cite{mhammedi2022efficient}, we know that for any $x\notin \K$
$$
\nabla \gamma(x)=\argmax_{s\in \hat{\K}} s^{\top} x
$$
where $\hat{\K}$ is the polar set of $\K$ defined as
$$
\hat{\K}=\{s\in \reals^d| s^{\top} x\le 1, \forall x\in \K\}
$$
From this we know the magnitude has to be $\frac{1}{\Pi_{\gamma}(x)^{\top} v(\Pi_{\gamma}(x))}$, and the gradient $\nabla \gamma(x)$ is really just 
$$
\frac{v(\Pi_{\gamma}(x))}{\Pi_{\gamma}(x)^{\top} v(\Pi_{\gamma}(x))}
$$
\end{proof}

\subsection{Proof of Lemma \ref{smooth}}
\begin{proof}
From Lemma \ref{zaklemma1} we already know one can use $\log_2 (\frac{2D^3}{r^4 \epsilon})$ calls to the membership oracle to get an $\epsilon$-approximation to $\gamma(x)$. We now use it to estimate $\nabla \gamma(x)$: for any $i\in [d]$ and a small constant $\lambda$, we have the following inequality for some $x'$.
$$
|\gamma(x+\lambda e_i)-\gamma(x)- \lambda \nabla\gamma(x)_i|\le \frac{\lambda^2}{2}\|\nabla^2 \gamma(x') \|_2^2
$$
From the smoothness assumption we know that $\|\nabla^2 \gamma(x') \|_2^2\le \beta^2$. Then the estimation error of $\nabla \gamma(x)_i$ can be upper bounded by $\frac{\lambda \beta^2}{2}+\frac{2\delta}{\lambda}$ where $\delta$ is the precision of $\gamma(x)$. Take $\delta=\frac{1}{dT^5}$ and $\lambda=\frac{1}{\sqrt{d} T^{2.5}}$, the estimation error to $\nabla \gamma(x)$ is bounded by $\frac{1}{\sqrt{d}T^2}$ while we only use $O(d\log T)$ calls to the membership oracle by Lemma \ref{zaklemma1}, that's $O(\log T)$ calls for each coordinate $i$. Since the estimation error for each coordinate is now bounded by $\frac{1}{\sqrt{d}T^2}$, the overall error is bounded by $$\sqrt{d}\times \frac{1}{\sqrt{d}T^2}=\frac{1}{T^2}$$

\end{proof}

%% file: B-set-smoothing.tex
\section{Set Smoothing for Polytopes}

Let $\K \subset \reals^d$ be defined as a level set of a function $h$, 
$$ \K = \{ x  \in \reals^d | h(x) \leq 0\} . $$
where $h(x)=\max_{i=1}^m (x^{\top} \alpha_i+b_i )$ and each $\alpha_i$ is a unit vector. Consider the smoothed set 
$$
\K_a=\{ x  \in \reals^d | h_a(x) \leq 0\} 
$$
where $h_a(x)$ is defined as 
$$
h_a(x)=\frac{1}{a}\log (\sum_{i=1}^m e^{a (x^{\top} \alpha_i+b_i)})
$$

We additionally assume that we have access to the linear constraints of $h(x)$, so that the oracle to $\K$ is stronger than just membership, but still weaker than projection.

We will show several properties of this smoothing method in order to determine the value of $a$. First, the new function $h_a(x)$ is $O(m a )$-smooth. Second, we notice that $\K_a$ is a subset of $\K$, while it also contains
$$ \{ x  \in \reals^d | h(x) \leq -\frac{\log m}{a}\}  $$
because $$h_a(x)\le \frac{1}{a}\log (m e^{a h(x)})\le h(x)+\frac{\log m}{a}.$$ In addition, for any $x$ such that $h(x)=0$, there exists $x'$ such that $h(x) \leq -\frac{\log m}{a}$ and $\|x-x'\|_2\le \frac{D\log m}{r a}$, simply by shrinking $x$ along the same direction and using the geometric assumptions on $\K$.

Now we choose $a=T^3$. Because the loss function is Lipschitz, replacing $\K$ by $\K_a$ can only hurt the regret by $O(T \tilde{G} \frac{D\log m}{rT^3})=O(\frac{1}{T^2})$ which is negligible, using a similar argument as in the proof of Theorem \ref{theorem: main}. 

Meanwhile, the smoothness constant being now $O(T^3)$ isn't an issue because we can multiply an additional $\frac{1}{T^3}$ factor to the method in Lemma \ref{smooth} to get a similar guarantee, i.e. choosing $\delta=\frac{1}{dT^{11}}$ and $\lambda=\frac{1}{\sqrt{d} T^{5.5}}$ using still $O(d \log T)$ membership oracle calls. 

Finally, because $h_a(x)$ has a clear closed form and can be exactly computed by $m$ computations of the linear constraints in $h(x)$, we can implement an exact membership oracle to $\K_a$ using $m$ times more calls. In all, we can achieve the same guarantee as in Lemma \ref{smooth} using $O(m d\log T)$ number of calls to the membership oracle instead.

%% file: main.bbl
\begin{thebibliography}{JOWW17}

\bibitem[BCHM20]{brukhim2020online}
Nataly Brukhim, Xinyi Chen, Elad Hazan, and Shay Moran.
\newblock Online agnostic boosting via regret minimization.
\newblock {\em Advances in Neural Information Processing Systems}, 33:644--654,
  2020.

\bibitem[BH21]{brukhim2021online}
Nataly Brukhim and Elad Hazan.
\newblock Online boosting with bandit feedback.
\newblock In {\em Algorithmic Learning Theory}, pages 397--420. PMLR, 2021.

\bibitem[CM78]{cobzas1978norm}
S~Cobzas and C~Mustata.
\newblock Norm-preserving extension of convex lipschitz functions.
\newblock {\em J. Approx. Theory}, 24(3):236--244, 1978.

\bibitem[Cut20]{cutkosky2020parameter}
Ashok Cutkosky.
\newblock Parameter-free, dynamic, and strongly-adaptive online learning.
\newblock In {\em International Conference on Machine Learning}, pages
  2250--2259. PMLR, 2020.

\bibitem[DGSS15]{daniely2015strongly}
Amit Daniely, Alon Gonen, and Shai Shalev-Shwartz.
\newblock Strongly adaptive online learning.
\newblock In {\em International Conference on Machine Learning}, pages
  1405--1411. PMLR, 2015.

\bibitem[FW56]{frank1956algorithm}
Marguerite Frank and Philip Wolfe.
\newblock An algorithm for quadratic programming.
\newblock {\em Naval research logistics quarterly}, 3(1-2):95--110, 1956.

\bibitem[GH15]{garber2015faster}
Dan Garber and Elad Hazan.
\newblock Faster rates for the frank-wolfe method over strongly-convex sets.
\newblock In {\em International Conference on Machine Learning}, pages
  541--549. PMLR, 2015.

\bibitem[GK22]{garber2022new}
Dan Garber and Ben Kretzu.
\newblock New projection-free algorithms for online convex optimization with
  adaptive regret guarantees.
\newblock {\em arXiv preprint arXiv:2202.04721}, 2022.

\bibitem[Haz16]{hazan2016introduction}
Elad Hazan.
\newblock Introduction to online convex optimization.
\newblock {\em Foundations and Trends{\textregistered} in Optimization},
  2(3-4):157--325, 2016.

\bibitem[HK12]{hazan2012projection}
Elad Hazan and Satyen Kale.
\newblock Projection-free online learning.
\newblock In {\em Proceedings of the 29th International Coference on
  International Conference on Machine Learning}, pages 1843--1850, 2012.

\bibitem[HM20]{hazan2020faster}
Elad Hazan and Edgar Minasyan.
\newblock Faster projection-free online learning.
\newblock In {\em Conference on Learning Theory}, pages 1877--1893. PMLR, 2020.

\bibitem[HS09]{hazan2009efficient}
Elad Hazan and Comandur Seshadhri.
\newblock Efficient learning algorithms for changing environments.
\newblock In {\em Proceedings of the 26th annual international conference on
  machine learning}, pages 393--400, 2009.

\bibitem[HS21]{pmlr-v139-hazan21a}
Elad Hazan and Karan Singh.
\newblock Boosting for online convex optimization.
\newblock In Marina Meila and Tong Zhang, editors, {\em Proceedings of the 38th
  International Conference on Machine Learning}, volume 139 of {\em Proceedings
  of Machine Learning Research}, pages 4140--4149. PMLR, 18--24 Jul 2021.

\bibitem[JOWW17]{jun2017improved}
Kwang-Sung Jun, Francesco Orabona, Stephen Wright, and Rebecca Willett.
\newblock Improved strongly adaptive online learning using coin betting.
\newblock In {\em Artificial Intelligence and Statistics}, pages 943--951.
  PMLR, 2017.

\bibitem[KG21]{kretzu2021revisiting}
Ben Kretzu and Dan Garber.
\newblock Revisiting projection-free online learning: the strongly convex case.
\newblock In {\em International Conference on Artificial Intelligence and
  Statistics}, pages 3592--3600. PMLR, 2021.

\bibitem[LH22]{lu2022efficient}
Zhou Lu and Elad Hazan.
\newblock Efficient adaptive regret minimization.
\newblock {\em arXiv preprint arXiv:2207.00646}, 2022.

\bibitem[LK19]{levy2019projection}
Kfir Levy and Andreas Krause.
\newblock Projection free online learning over smooth sets.
\newblock In {\em The 22nd International Conference on Artificial Intelligence
  and Statistics}, pages 1458--1466. PMLR, 2019.

\bibitem[LSV18]{lee2018efficient}
Yin~Tat Lee, Aaron Sidford, and Santosh~S Vempala.
\newblock Efficient convex optimization with membership oracles.
\newblock In {\em Conference On Learning Theory}, pages 1292--1294. PMLR, 2018.

\bibitem[LXAH22]{lu2022adaptive}
Zhou Lu, Wenhan Xia, Sanjeev Arora, and Elad Hazan.
\newblock Adaptive gradient methods with local guarantees.
\newblock {\em arXiv preprint arXiv:2203.01400}, 2022.

\bibitem[Mha22]{mhammedi2022efficient}
Zakaria Mhammedi.
\newblock Efficient projection-free online convex optimization with membership
  oracle.
\newblock In {\em Conference on Learning Theory}, pages 5314--5390. PMLR, 2022.

\bibitem[WZ21]{wan2021projection}
Yuanyu Wan and Lijun Zhang.
\newblock Projection-free online learning over strongly convex sets.
\newblock In {\em Proceedings of the AAAI Conference on Artificial
  Intelligence}, volume~35, pages 10076--10084, 2021.

\end{thebibliography}
